\documentclass{article}

\usepackage[final,main]{neurips_2026}
\usepackage{hyperref}
\usepackage[utf8]{inputenc}
\usepackage[T1]{fontenc}
\usepackage{microtype}
\usepackage{url}
\usepackage{booktabs}
\usepackage{amsmath,amssymb,amsfonts,amsthm,mathtools}
\usepackage{graphicx}
\usepackage{xcolor}

\newif\ifshowcamerachanges
\showcamerachangesfalse
\newcommand{\camchange}[1]{\ifshowcamerachanges\textcolor{blue}{#1}\else#1\fi}
\newenvironment{camchangeblock}{\ifshowcamerachanges\color{blue}\fi}{}
\usepackage{hyperref}
\hypersetup{hidelinks}

\title{Spectral Graph Sparsification Preserves Representation Geometry in Graph Neural Networks}

\author{%
  Sanjukta Krishnagopal \\
  Department of Computer Science, University of California Santa Barbara \\
  \texttt{sanjukta@ucsb.edu}
}

\newtheorem{theorem}{Theorem}
\newtheorem{proposition}{Proposition}
\newtheorem{lemma}{Lemma}
\newtheorem{corollary}{Corollary}

\begin{document}

\maketitle

\begin{abstract}
Spectral graph sparsification is a classical tool for reducing graph complexity while preserving Laplacian quadratic forms. In graph neural networks (GNNs), sparsification is often used to accelerate computation while maintaining predictive performance. In this work, we study a complementary representation-level question: does sparsification preserve the geometry of learned embeddings?

For polynomial-filter GNNs, we prove that any $\epsilon$-spectral sparsifier induces $O(\epsilon)$ perturbations in polynomial graph filters, multilayer hidden representations, and their Gram matrices. These guarantees imply stability of squared pairwise distances, class means, and covariance structure in embedding space. We further establish finite-time training stability: under smoothness and boundedness assumptions, gradient descent on dense and sparsified graphs produces weight trajectories whose separation grows at most proportionally to the sparsification distortion.


Empirically, effective-resistance sparsification validates the predicted perturbation chain on synthetic graphs and preserves hidden representation geometry on real datasets. In our experiments, the Gram matrix and training dynamics show low divergence even under substantial sparsification, consistent with the predicted stability under spectral sparsification. Hidden Gram preservation strongly predicts neighborhood preservation and class-centroid stability across FashionMNIST, Cora, and Paul15. Together, these results show that spectral sparsification preserves not only graph operators, but also the representation geometry that supports downstream use of GNN embeddings for interpretability.
\end{abstract}

\section{Introduction}
Graph neural networks (GNNs) are a standard framework for learning from relational data, with applications in citation networks, molecules, recommendation systems, and scientific data analysis
\citep{kipf2017gcn,hamilton2017graphsage,defferrard2016cnn}. Their effectiveness, however, comes with a scalability cost: training and inference often depend strongly on the size and density of the underlying graph. This is especially true for architectures whose propagation rules are expressed through graph operators, such as spectral and polynomial-filter GNNs~\citep{bruna2014spectral,defferrard2016cnn}.

Spectral graph sparsification offers a principled way to reduce this cost. Classical sparsification theory shows that a graph can be replaced by a much sparser graph whose Laplacian $\widetilde L$ satisfies
\[
(1-\varepsilon)L \preceq \widetilde L \preceq (1+\varepsilon)L,
\]
thereby preserving quadratic forms and spectral structure up to controlled distortion
\citep{spielman2011graph}. Recent work has begun to use such ideas for efficient GNNs, including effective-resistance- and Laplacian-based sparsification methods for spectral and attention-style architectures
\citep{ding2025largescale}. Most of this literature evaluates sparsification through operator approximation, runtime, or predictive accuracy.

In this paper, we ask a representation-centered question: \emph{does spectral sparsification preserve the geometry of learned hidden representations?} This question is important because GNNs are not only predictors; they are also representation learners
\citep{hamilton2017representation}. Their hidden states encode pairwise similarities, class separation, and global latent geometry. Two sparse graphs may yield similar accuracy while inducing substantially different embeddings, which can affect downstream uses such as retrieval, clustering, transfer learning, and interpretation. Recent work on feature evolution and neural collapse in GNNs further emphasizes that graph topology can shape representation geometry during training
\citep{kothapalli2023neuralcollapsegnn,papyan2020neuralcollapse}.
We focus on polynomial-filter GNNs, a broad class of spectral graph neural networks in which each layer applies a finite-degree polynomial in a graph operator before feature mixing and nonlinear activation \citep{defferrard2016cnn}. This class includes Chebyshev and other finite-hop spectral filters, and it also captures first-order propagation layers when the graph shift is written as a degree-one polynomial operator. \camchange{In particular, ChebNet is explicitly polynomial and a standard GCN layer is degree one in the graph shift, so both fall within the analyzed class; architectures with learned or feature-dependent graph operators, such as attention-based propagation, are not directly covered by the present guarantees.} It is therefore a natural setting for studying spectral sparsification, since the sparsifier directly perturbs the graph operator appearing in the layer.
Building on prior analyses of GNN stability and transferability under graph perturbations \citep{gama2020stability,ruiz2021gnnprocieee,ruiz2023transferability}, we move beyond operator-level guarantees and analyze hidden representation preservation directly. We prove that spectral sparsification stabilizes multilayer embeddings and their Gram matrices, yielding preservation of pairwise distances, class means, and covariance structure. We also analyze learning dynamics under gradient descent in a one-layer model and give a multilayer extension under stronger smoothness and boundedness assumptions, connecting to recent work on GNN training behavior on growing graphs \citep{krishnagopal2023gntk}.

\paragraph{Contributions.}
\begin{itemize}
    \item We introduce a representation-geometric perspective on graph sparsification for GNNs.
    \item We prove multilayer embedding stability and Gram-matrix stability under spectral sparsification.
    \item We establish finite-time training-dynamics stability in a one-layer setting, and a multilayer extension under stronger smoothness assumptions.
    \item We show empirically, using effective-resistance sparsification as a concrete testbed, on  synthetic and real data, that hidden-Gram preservation predicts neighborhood and class-geometry preservation in learned embeddings, improving interpretability and propagation fidelity.
\end{itemize}

\paragraph{Relation to prior work.}
Our work connects several lines of research. First, spectral graph sparsification provides classical
operator-level guarantees for replacing a graph by a sparse spectral approximation
\citep{spielman2011graph}. Second, recent sparsification-for-GNN methods use these ideas to
improve efficiency, focusing mainly on propagation fidelity and predictive performance
\citep{ding2025largescale}. Third, prior work studies the stability and transferability of GNNs
under graph perturbations, graph shifts, and graph-size changes
\citep{gama2020stability,ruiz2021gnnprocieee,ruiz2023transferability}. Fourth, recent work
studies the geometry of learned features in graph representation learning, including neural collapse
in GNNs \citep{kothapalli2023neuralcollapsegnn,papyan2020neuralcollapse}. Finally, our viewpoint is
also related to large-graph asymptotic analyses of GNN transferability and learning dynamics
\citep{krishnagopal2023gntk}. Our contribution lies at the intersection of these directions: we
study whether principled spectral sparsification preserves the geometry of learned representations,
not only the propagation operator or task accuracy.

\section{Preliminaries and Main Results}
\label{sec:main_results}

\subsection{Setup and Notation}

Let $G=(V,E)$ be a graph on $n$ nodes, and let $L \in \mathbb{R}^{n\times n}$ denote a symmetric positive semidefinite graph operator, such as the combinatorial or normalized Laplacian. For symmetric matrices $A,B \in \mathbb{R}^{n\times n}$, we write $A \preceq B$ if $B-A$ is positive semidefinite. We use $\|\cdot\|_2$ for the operator norm and $\|\cdot\|_F$ for the Frobenius norm.

We say that $\widetilde L$ is an $\varepsilon$-spectral sparsifier of $L$ if
\begin{equation}
(1-\varepsilon)L \preceq \widetilde L \preceq (1+\varepsilon)L,
\qquad 0<\varepsilon<1.
\label{eq:spectral_sparsifier}
\end{equation}

We study a class of $K$-layer polynomial filter graph neural networks of the form
\begin{equation}
H^{(0)} = X, \qquad
H^{(k+1)} = \sigma_k\!\bigl(p_k(L)H^{(k)}W_k\bigr),
\qquad k=0,\dots,K-1,
\label{eq:gnn_forward}
\end{equation}
where $X \in \mathbb{R}^{n\times d_0}$ is the input feature matrix, $H^{(k)} \in \mathbb{R}^{n\times d_k}$ is the hidden representation at layer $k$, $W_k \in \mathbb{R}^{d_k\times d_{k+1}}$ is the trainable weight matrix, $p_k$ is a real polynomial, and $\sigma_k$ is applied entrywise.

The corresponding network on the sparsified graph is defined by
\begin{align}
\widetilde H^{(0)} &= X, \notag\\
\widetilde H^{(k+1)}
&= \sigma_k\!\bigl(p_k(\widetilde L)\widetilde H^{(k)}W_k\bigr).
\label{eq:sparse_gnn_forward}
\end{align}
We denote the final embeddings by
\[
Z := H^{(K)},
\qquad
\widetilde Z := \widetilde H^{(K)}.
\]

Our analysis is based on the following assumptions.

\textbf{(A1) Polynomial filters.}
For each $k$, $p_k(x)=\sum_{r=0}^{d_k^\star} a_{k,r}x^r$ is a polynomial of finite degree $d_k^\star$.

\textbf{(A2) Lipschitz nonlinearities.}
Each activation $\sigma_k:\mathbb{R}\to\mathbb{R}$ is $L_{\sigma_k}$-Lipschitz.

\textbf{(A3) Bounded weights.}
There exists $B_W>0$ such that $\|W_k\|_2 \le B_W$ for all $k$.

\textbf{(A4) Bounded inputs.}
The input features satisfy $\|X\|_F \le B_X$.

\textbf{(A5) Bounded operators.}
There exists $B_L>0$ such that $\|L\|_2,\|\widetilde L\|_2 \le B_L$.

Assumptions (A1)--(A5) are standard in perturbation analyses of graph-based learning systems and are sufficient for the stability results below. \camchange{They match the finite-degree filter, Lipschitz nonlinearity, and bounded-operator conditions commonly used in GNN stability and transferability analyses~\citep{gama2020stability,ruiz2021gnnprocieee,ruiz2023transferability}.}

\subsection{Propagation Stability}

We first show that spectral sparsification preserves polynomial graph filters.

\begin{theorem}[Polynomial filter stability]
\label{thm:poly_stability}
Let $p(x)=\sum_{r=0}^{d} a_r x^r$ be a polynomial, and let $L,\widetilde L$ satisfy \eqref{eq:spectral_sparsifier}. Under (A5),
\begin{equation}
\|p(L)-p(\widetilde L)\|_2 \le C_p\,\varepsilon,
\label{eq:poly_stability}
\end{equation}
where
\[
C_p := \sum_{r=1}^{d} |a_r|\, r\, B_L^{\,r}.
\]
\end{theorem}

Theorem~\ref{thm:poly_stability} is the operator-level ingredient in our analysis: it says that spectral closeness of the graph implies closeness of the propagation operators used by polynomial-filter GNNs. The full proof is given in Appendix~\ref{app:poly_stability}.

\subsection{Representation Stability}

We next propagate the operator perturbation through a multilayer nonlinear architecture.

\begin{theorem}[Multilayer representation stability]
\label{thm:multilayer_stability}
Suppose that \eqref{eq:spectral_sparsifier} holds and assumptions (A1)--(A5) are satisfied. Then there exists a constant $C_{\mathrm{rep}}>0$, depending only on the network depth, filter coefficients, Lipschitz constants, and bounds $B_X,B_W,B_L$, such that
\begin{equation}
\|Z-\widetilde Z\|_F \le C_{\mathrm{rep}}\,\varepsilon.
\label{eq:multilayer_stability}
\end{equation}
\end{theorem}

\begin{camchangeblock}
The constant in Theorem~\ref{thm:multilayer_stability} can be written explicitly from the layerwise recursion in Appendix~\ref{app:multilayer_stability}. If
\[
\alpha_k=L_{\sigma_k}B_W C_{p_k}B_k',\qquad
\beta_k=L_{\sigma_k}B_W M_k,
\]
then one may take
\[
C_{\mathrm{rep}}=\sum_{j=0}^{K-1}\alpha_j\prod_{\ell=j+1}^{K-1}\beta_\ell.
\]
This also makes the graph-size dependence explicit. For bounded per-node features of fixed dimension, $\|X\|_F=O(\sqrt n)$. With a normalized or scaled Laplacian ($B_L=O(1)$), fixed depth, and fixed filter degrees, the forward bounds satisfy $B_k'=O(\sqrt n)$ and hence $\|Z-\widetilde Z\|_F=O(\sqrt n\,\varepsilon)$. Thus there is no hidden exponential dependence on graph size. For an unnormalized combinatorial Laplacian the dependence remains polynomial through $B_L=\|L\|_2$.
\end{camchangeblock}

The result shows that spectral sparsification controls not only the graph operator but also the learned embedding itself. The full proof is given in Appendix~\ref{app:multilayer_stability}.

\subsection{Representation Geometry}

The main conceptual contribution of this paper is that sparsification preserves the geometry of the learned embedding space.

\begin{theorem}[Gram matrix stability]
\label{thm:gram_stability}
Under the assumptions of Theorem~\ref{thm:multilayer_stability}, there exist constants $C_{\mathrm{gram},2},C_{\mathrm{gram},F}>0$ such that
\[
\|ZZ^\top - \widetilde Z\widetilde Z^\top\|_2
\le C_{\mathrm{gram},2}\,\varepsilon,
\qquad
\|ZZ^\top - \widetilde Z\widetilde Z^\top\|_F
\le C_{\mathrm{gram},F}\,\varepsilon.
\]
\end{theorem}

Theorem~\ref{thm:gram_stability} upgrades representation stability to Gram-matrix stability in both spectral and Frobenius norms. Since the Gram matrix encodes all pairwise inner products between node embeddings, this implies preservation of distances and class-level statistics such as centroids and covariances in latent space. \camchange{The proof is the direct factorization $ZZ^\top-\widetilde Z\widetilde Z^\top=Z(Z-\widetilde Z)^\top+(Z-\widetilde Z)\widetilde Z^\top$, followed by the bounded forward-state estimates from Theorem~\ref{thm:multilayer_stability}. Under the normalized/scaled graph-size regime described above, $\|Z\|_F,\|\widetilde Z\|_F=O(\sqrt n)$, so the corresponding worst-case Gram perturbation is $O(n\varepsilon)$ rather than exponential in $n$.} The full proof is given in Appendix~\ref{app:gram_stability}.

The following corollaries formalize this implication.

\begin{corollary}[Pairwise distance preservation]
\label{cor:distance_preservation}
Let $z_i^\top$ and $\widetilde z_i^\top$ denote the $i$th rows of $Z$ and $\widetilde Z$, respectively. Then for all $i,j \in \{1,\dots,n\}$,
\begin{equation}
\bigl|
\|z_i-z_j\|_2^2 - \|\widetilde z_i-\widetilde z_j\|_2^2
\bigr|
\le 4 C_{\mathrm{gram},2}\,\varepsilon.
\label{eq:distance_preservation}
\end{equation}
\end{corollary}

Now suppose that the nodes are partitioned into $C$ classes with index sets $S_1,\dots,S_C$, where $n_c := |S_c|$. Define the class mean and empirical covariance for class $c$ by
\begin{equation}
\mu_c := \frac{1}{n_c}\sum_{i\in S_c} z_i, \qquad
\Sigma_c := \frac{1}{n_c}\sum_{i\in S_c}(z_i-\mu_c)(z_i-\mu_c)^\top,
\label{eq:class_cov}
\end{equation}
and define $\widetilde\mu_c,\widetilde\Sigma_c$ analogously using $\widetilde Z$.

\begin{corollary}[Class-statistics stability]
\label{cor:class_stats}
Under the assumptions of Theorem~\ref{thm:multilayer_stability}, for each class $c$,
\begin{equation}
\|\mu_c-\widetilde\mu_c\|_2
\le \frac{1}{\sqrt{n_c}}\|Z-\widetilde Z\|_F,
\label{eq:mean_stability}
\end{equation}
and, if the row norms of $Z$ and $\widetilde Z$ are uniformly bounded by $B_Z$,
\begin{equation}
\|\Sigma_c-\widetilde\Sigma_c\|_2
\le \frac{8B_Z}{\sqrt{n_c}}\|Z-\widetilde Z\|_F.
\label{eq:cov_stability}
\end{equation}
In particular, both quantities are $O(\varepsilon)$.
\end{corollary}

Corollary~\ref{cor:class_stats} controls several representation diagnostics used in practice. Examples include the displacement of class centroids, pairwise distances and angles between centroids, within-class variance $\operatorname{tr}(\Sigma_c)$, leading covariance eigenvalues, between-/within-class scatter ratios, and nearest-centroid margins. These quantities measure whether classes remain separated, compact, and similarly oriented after sparsification. Neural-collapse-style diagnostics are also covered, since they are built from the same objects: class means, within-class covariances, and normalized centroid geometry. On bounded sets, any such diagnostic that is locally Lipschitz as a function of
$
\{\mu_c,\Sigma_c\}_{c=1}^C
$
therefore changes by at most $O(\varepsilon)$ under spectral sparsification.

Proofs of Corollaries~\ref{cor:distance_preservation} and~\ref{cor:class_stats} are collected in Appendix~\ref{app:geometry_corollaries}.

Theorems~\ref{thm:poly_stability}--\ref{thm:gram_stability} establish the implication chain
\[
L \approx \widetilde L
\;\Longrightarrow\;
p(L) \approx p(\widetilde L)
\;\Longrightarrow\;
Z \approx \widetilde Z
\;\Longrightarrow\;
ZZ^\top \approx \widetilde Z\widetilde Z^\top.
\]

\subsection{Training Dynamics}
We also study how sparsification perturbs training trajectories. In a one-layer model
\[
f(W;L)=\sigma\!\bigl(p(L)XW\bigr), \qquad
J(W;L)=\tfrac12\|f(W;L)-Y\|_F^2,
\]

In addition to (A1)--(A5), we make the following standard assumptions:

\textbf{(A6) Smooth activation.}
$\sigma \in C^1(\mathbb{R})$, $\sigma'$ is bounded, and $\sigma'$ is Lipschitz.

\textbf{(A7) Bounded iterates.}
The gradient descent iterates remain in a bounded set:
\[
\|W_t\|_2,\|\widetilde W_t\|_2 \le R
\qquad\text{for all } t\le T.
\]
\camchange{Assumptions (A6)--(A7) are finite-time regularity conditions used to control perturbations of the gradient field: smoothness makes the gradient locally Lipschitz, while bounded iterates restrict the analysis to the finite trajectory actually traversed by gradient descent.}
\textbf{(A8) Bounded targets.}
The target matrix satisfies $\|Y\|_F \le B_Y$.

Let
\[
W_{t+1}=W_t-\eta \nabla_W J(W_t;L), \qquad
\widetilde W_{t+1}=\widetilde W_t-\eta \nabla_W J(\widetilde W_t;\widetilde L),
\]
with common initialization $W_0=\widetilde W_0$.

\begin{theorem}[One-layer training stability]
\label{thm:training_stability}
Under these assumptions, with constants allowed to depend on $B_Y$, there exist constants $L_W,L_S>0$ such that for all $t\le T$,
\begin{equation}
\label{eq:theorem_training_one_layer_main}
\|W_{t+1}-\widetilde W_{t+1}\|_F
\le
(1+\eta L_W)\|W_t-\widetilde W_t\|_F
+\eta L_S\|p(L)-p(\widetilde L)\|_2.
\end{equation}

Consequently, by Theorem~\ref{thm:poly_stability},
\begin{equation}
\|W_t-\widetilde W_t\|_F
\le
C_{\mathrm{train}}\varepsilon \sum_{j=0}^{t-1}(1+\eta L_W)^j,
\end{equation}
for a constant $C_{\mathrm{train}}$ independent of $\varepsilon$.
\end{theorem}
Theorem~\ref{thm:training_stability} shows that sparsification perturbs the gradient field only slightly, and hence yields finite-time stability of gradient descent trajectories in the one-layer setting. The full proof is deferred to Appendix~\ref{app:training_stability}.

\paragraph{Multilayer training dynamics.}
The one-layer result extends to full-batch multilayer gradient descent on bounded finite-time trajectories under analogous smoothness and boundedness assumptions. Let
\[
\Theta:=(W_0,\dots,W_{K-1}), \qquad
F(\Theta;L):=H^{(K)},
\]
where
\[
H^{(0)}=X, \qquad
H^{(k+1)}=\sigma_k\!\bigl(p_k(L)H^{(k)}W_k\bigr),
\]
and let
\[
J(\Theta;L):=\tfrac12\|F(\Theta;L)-Y\|_F^2.
\]
In addition to (A1)--(A5) and (A8), we make the following standard assumptions:
\begin{itemize}
    \item[(A6$'$)] \textbf{Smooth activations.}
    For each $k$, $\sigma_k \in C^1(\mathbb{R})$, $\sigma_k'$ is bounded, and $\sigma_k'$ is Lipschitz.
    \item[(A7$'$)] \begin{camchangeblock}\textbf{Bounded iterates.}
    There exists $R>0$ such that for all $t\le T$ and all $k$,
    \[
    \|W_k^{\,t}\|_2,\ \|\widetilde W_k^{\,t}\|_2 \le R.
    \]
    Under (A1)--(A5), this weight bound implies uniform finite-time bounds on the hidden states by the same forward recursion used in Appendix~\ref{app:multilayer_stability}; hidden-state boundedness therefore need not be imposed as an independent assumption.\end{camchangeblock}
\end{itemize}
Let $\Theta^t,\widetilde\Theta^t$ be the corresponding gradient descent iterates with common initialization. We assume smoothness and boundedness for the activation functions and weights, ensuring that perturbations do not grow uncontrollably over finite training steps.

\begin{proposition}[Multilayer training stability]
\label{prop:multilayer_training_stability}
Under these assumptions, there exist constants $L_\Theta,L_S>0$ such that for all $t\le T$,
\begin{equation}
\label{eq:theorem_training_multi_layer_main}
\|\Theta^{t+1}-\widetilde\Theta^{t+1}\|_F
\le
(1+\eta L_\Theta)\|\Theta^t-\widetilde\Theta^t\|_F
+
\eta L_S \sum_{k=0}^{K-1}\|p_k(L)-p_k(\widetilde L)\|_2.
\end{equation}

Consequently, by Theorem~\ref{thm:poly_stability},
\[
\|\Theta^t-\widetilde\Theta^t\|_F
\le
\eta L_S
\left(\sum_{k=0}^{K-1} C_{p_k}\right)
\varepsilon
\sum_{j=0}^{t-1}(1+\eta L_\Theta)^j
\le
C_{\mathrm{train}}^{(K)}\varepsilon
\sum_{j=0}^{t-1}(1+\eta L_\Theta)^j,
\]
where $C_{\mathrm{train}}^{(K)}$ is independent of $\varepsilon$.

\end{proposition}

The proof (see \ref{app:multilayer_training_stability}) follows the same strategy as Theorem~\ref{thm:training_stability}: one shows that the full gradient field is Lipschitz in the parameters and Lipschitz in the collection of graph filters $\{p_k(L)\}_{k=0}^{K-1}$ on bounded sets, and then combines these bounds with Theorem~\ref{thm:poly_stability}.



\section{Experiments}
\label{sec:experiments}

We evaluate the central claim of the paper: in practically relevant sparsification regimes, principled graph sparsification preserves not only graph operators, but also the geometry of learned hidden representations. Our experiments combine controlled synthetic tests, which directly probe the perturbation predictions of Section~\ref{sec:main_results}, with practical representation-learning tests on real datasets.

\subsection{Setup}
\label{subsec:exp_setup_main}

We evaluate sparsification on synthetic and real graph datasets. The synthetic experiments use weighted stochastic block models (SBMs) and geometric $k$-NN graphs constructed to exhibit clear community or manifold structure. The real-data experiments use FashionMNIST \citep{xiao2017fashion}, Cora \citep{sen2008collective}, and Paul15 \citep{paul2015transcriptional}, a single-cell RNA-seq dataset profiling myeloid progenitor cells. \camchange{The full real-data graphs contain 2,200 FashionMNIST nodes, 2,708 Cora nodes, and 2,730 Paul15 cells.} Node features are either class-structured synthetic features, for controlled experiments, or the standard features provided with each dataset; details are given in Appendix~\ref{app:datasets}.

For the theory-aligned experiments in Figures~\ref{fig:theory_main} and~\ref{fig:training_dynamics_main}, we use deterministic GNN models to isolate the effect of graph sparsification. Specifically, we use a polynomial graph filter followed by two linear feature-mixing layers for forward-map evaluation, and one- and two-layer deterministic GNNs for training-dynamics experiments. Dense and sparsified models have matched initialization. \camchange{This deterministic setup is an experimental control, not a restriction of the theory: it isolates the change $L\mapsto\widetilde L$ from dropout, minibatching, or unmatched randomness. Standard GCNs remain within the polynomial-filter class analyzed above.}

Sparsification is performed using effective-resistance sampling. Effective-resistance sparsification samples edges with probability proportional to their leverage scores $w_e R_e$, where $w_e$ is the edge weight and $R_e$ is the effective resistance between the endpoints of edge $e$, and then reweights sampled edges so that the sparse Laplacian remains a spectrally accurate approximation of the original Laplacian~\citep{spielman2011graph,spielman2011spectral}.
This choice provides a concrete testbed for the theory, but our results apply to any method producing a graph $\widetilde L$ satisfying the spectral comparison~\eqref{eq:spectral_sparsifier}. The sparsification level is controlled by the edge-sampling budget, with larger budgets producing denser, lower-distortion sparsifiers. \camchange{Spectral sparsification is intended to be a computationally efficient preprocessing step: classical constructions are nearly linear in the original graph size and effective-resistance sampling yields sparse spectral approximations with $O(n\log n/\varepsilon^2)$ edges~\citep{spielman2011graph,spielman2011spectral}. Replacing $m$ edges by $\widetilde m$ edges reduces edge-dependent message-passing and storage costs from $O(md)$ to $O(\widetilde m d)$ per layer for feature width $d$. The generalized-eigenvalue calculation used below to report $\epsilon_{\mathrm{emp}}$ is only a theory-facing diagnostic and is not required by the sparsification or GNN training procedure.} Full architectural, optimization, and sparsification details are given in Appendices~\ref{app:impl} and~\ref{app:sparsifier}.

\subsection{Validation of representation stability}
\label{subsec:theory_validation_main}

\begin{table}[htbp!]
\centering
\setlength{\tabcolsep}{3pt}
\renewcommand{\arraystretch}{0.9}
\footnotesize
\begin{minipage}{0.47\linewidth}
\centering
\begin{tabular}{lcccccc}
\toprule
\multicolumn{7}{c}{Geometric} \\
\midrule
$\epsilon_{\mathrm{emp}}$ & 0.37 & 0.38 & 0.41 & 0.54 & 0.66 & 0.75 \\
frac.                     & 0.90 & 0.87 & 0.87 & 0.70 & 0.58 & 0.53 \\
\bottomrule
\end{tabular}
\end{minipage}
\hfill
\begin{minipage}{0.47\linewidth}
\centering
\begin{tabular}{lcccccc}
\toprule
\multicolumn{7}{c}{SBM} \\
\midrule
$\epsilon_{\mathrm{emp}}$ & 0.41 & 0.42 & 0.43 & 0.43 & 0.56 & 0.71 \\
frac.                     & 0.96 & 0.96 & 0.92 & 0.89 & 0.77 & 0.66 \\
\bottomrule
\end{tabular}
\end{minipage}
\caption{Mean empirical spectral distortion $\epsilon_{\mathrm{emp}}$ and retained edge fraction (frac.) for the synthetic sparsification experiments, averaged over $5$ sparsifier draws.}
\label{tab:eps_edgefrac_sidebyside}
\end{table}

\begin{figure}[htbp!]
    \centering
    \includegraphics[width=\linewidth]{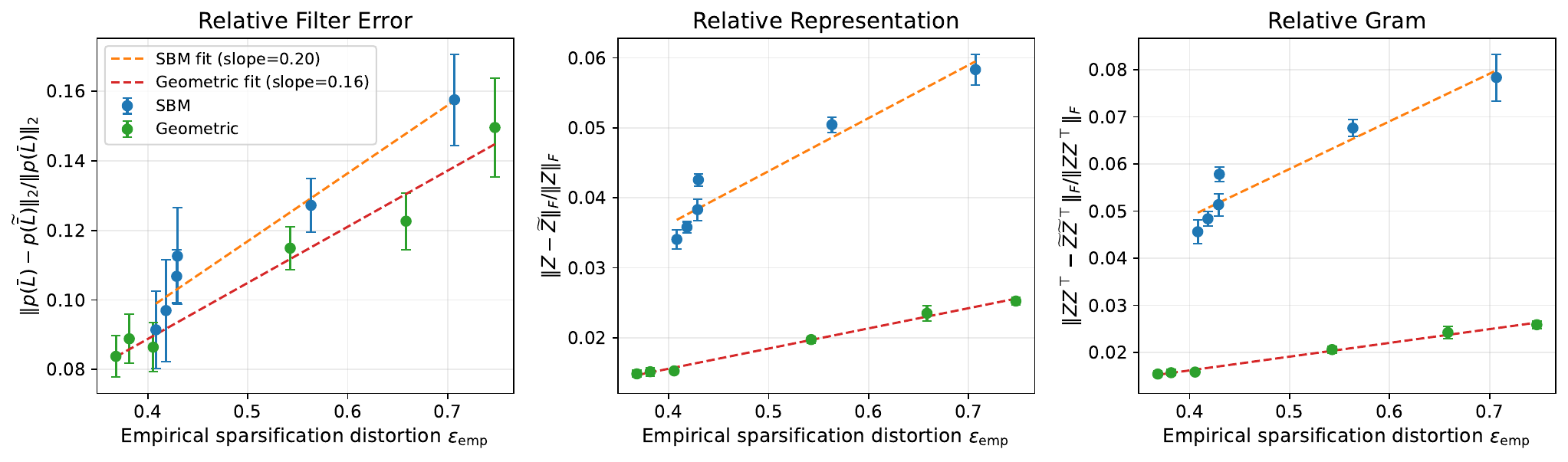}
    \caption{
\textbf{Controlled validation of representation stability.}
We measure relative polynomial-filter error, relative hidden-representation error, and relative hidden-Gram error for effective resistance sparsified weighted SBM and geometric weighted \(k\)-NN graphs. The graph operator is the scaled combinatorial Laplacian \(\bar L=L/\|L\|_2\), and the polynomial filter is \(p(\bar L)=I-0.6\bar L+0.15\bar L^2\). Across both graph families, all three quantities increase smoothly with \(\epsilon_{\mathrm{emp}}\), consistent with the stability predictions of Theorems~\ref{thm:poly_stability}--\ref{thm:gram_stability}. Markers show means over five sparsifier draws at each selected budget, and vertical error bars denote one standard deviation.
}
    \label{fig:theory_main}
\end{figure}

On weighted SBM and geometric graph families, we construct effective-resistance sparsifiers with varying degrees of sparsification, controlled by varying the sampling budget $q$. For synthetic graphs, we measure empirical spectral distortion $\epsilon_{\mathrm{emp}}$ from the generalized eigenvalue envelope of the dense and sparse Laplacians. If $\lambda_{\min}$ and $\lambda_{\max}$ are the extremal generalized eigenvalues of $(\widetilde L,L)$ on the nonconstant subspace, we compute
\[
\epsilon_{\mathrm{emp}}=\max\{1-\lambda_{\min},\lambda_{\max}-1\}.
\]
Table~\ref{tab:eps_edgefrac_sidebyside} reports the corresponding fraction of retained edges.

We then evaluate the quantities appearing in the theory: relative polynomial-filter error, relative hidden-representation error, and relative hidden-Gram error. The synthetic forward map uses the scaled Laplacian graph operator $\bar L=L/\|L\|_2$ and polynomial filter
\[
p(\bar L)=I-0.6\bar L+0.15\bar L^2.
\]
\camchange{The relative errors reported here are normalized versions of the absolute perturbations controlled in Section~\ref{sec:main_results}. Because the dense-graph denominator is fixed with respect to sparsification, this normalization preserves the predicted $O(\varepsilon)$ dependence whenever the denominator is nonzero.} Since the Gram matrix encodes pairwise inner products, stability of $ZZ^\top$ directly reflects preservation of embedding geometry.

Figure~\ref{fig:theory_main} validates the perturbation chain of Section~\ref{sec:main_results}. Across both graph families, all three quantities increase smoothly and approximately linearly with $\epsilon_{\mathrm{emp}}$, consistent with the $O(\epsilon)$ scaling predicted by Theorems~\ref{thm:poly_stability}--\ref{thm:gram_stability}. \camchange{The relative representation and Gram distortions remain small $(0.03-0.08)$ even in substantial-distortion regimes where only roughly half the edges remain ($\epsilon_{\mathrm{emp}} \sim 0.7$). We interpret these large-$\epsilon_{\mathrm{emp}}$ points as empirical stress tests rather than as a regime in which the perturbative constants are expected to be sharp.}

\begin{figure}[htbp!]
    \centering
    \includegraphics[width=\linewidth]{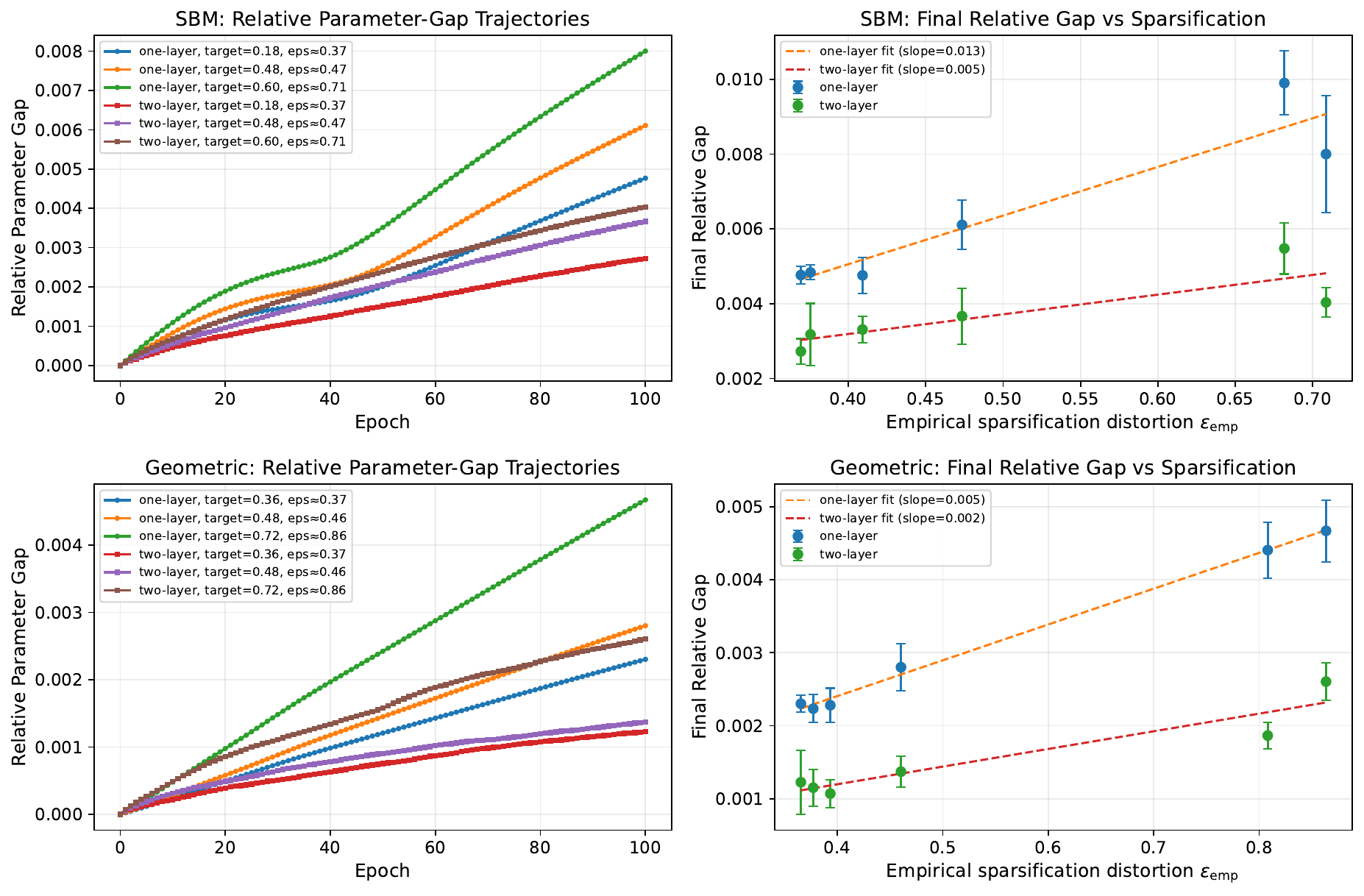}
\caption{
\textbf{Training-trajectory perturbation under graph sparsification.}
Relative parameter-gap trajectories between models trained on the dense graph and on effective-resistance sparsifiers for one-layer and two-layer deterministic polynomial-filter GNNs on SBM and geometric graph families. Both experiments use the same scaled polynomial-Laplacian graph shift as Figure~\ref{fig:theory_main}. Left: parameter-gap trajectories over training epochs for representative empirical distortion levels. Right: final relative parameter gap versus empirical sparsification distortion \(\epsilon_{\mathrm{emp}}\). Markers show means over five sparsifier draws at fixed target distortion, vertical error bars denote one standard deviation. The monotone increase with \(\epsilon_{\mathrm{emp}}\) is consistent with the finite-time perturbation picture of Theorem~\ref{thm:training_stability} and Proposition~\ref{prop:multilayer_training_stability}. The relative parameter gap is a trajectory-level diagnostic in parameter space and should not be interpreted as graph or representation distortion.
}
    \label{fig:training_dynamics_main}
\end{figure}

\subsection{Training dynamics under graph sparsification}
\label{subsec:training_dynamics_main}

We test whether the finite-time perturbation behavior predicted by
Theorem~\ref{thm:training_stability} and
Proposition~\ref{prop:multilayer_training_stability} is visible in controlled training experiments.
For each dense graph and sparsifier, we train dense-graph and sparse-graph models from the same initialization and compare their parameter trajectories over time. The relative parameter gap is defined as the Frobenius norm of the difference between the weight matrices of the models trained on the dense and sparsified graphs:
\[
\mathrm{Relative\ Parameter\ Gap}=
\frac{\|\Theta^t-\widetilde\Theta^t\|_F}{\|\Theta^t\|_F}.
\]
with the analogous one-layer definition when there is only one weight matrix. This is a trajectory-level diagnostic in parameter space. It should not be interpreted as graph or representation distortion.

Figure~\ref{fig:training_dynamics_main} shows that, across SBM and geometric graph families, the relative parameter gap grows smoothly over training and is consistently larger for sparsifiers with larger empirical spectral distortion \(\epsilon_{\mathrm{emp}}\). The final relative gap also increases approximately linearly with \(\epsilon_{\mathrm{emp}}\). This behavior is consistent with the recurrence in Theorem~\ref{thm:training_stability}: sparsification introduces an \(O(\epsilon)\) perturbation to the gradient field at each step, and these perturbations accumulate over finite training time.

In this polynomial-Laplacian setup, the absolute parameter gaps remain small. Thus, Figure~\ref{fig:training_dynamics_main} should be read as qualitative evidence for the predicted monotone finite-time dependence on sparsification distortion, while the main representation-level validation is provided by Figure~\ref{fig:theory_main} and the real-data geometry experiments below.

\subsection{Hidden-geometry preservation on real data}

\begin{figure}[htbp!]
    \centering
    \includegraphics[width=\linewidth]{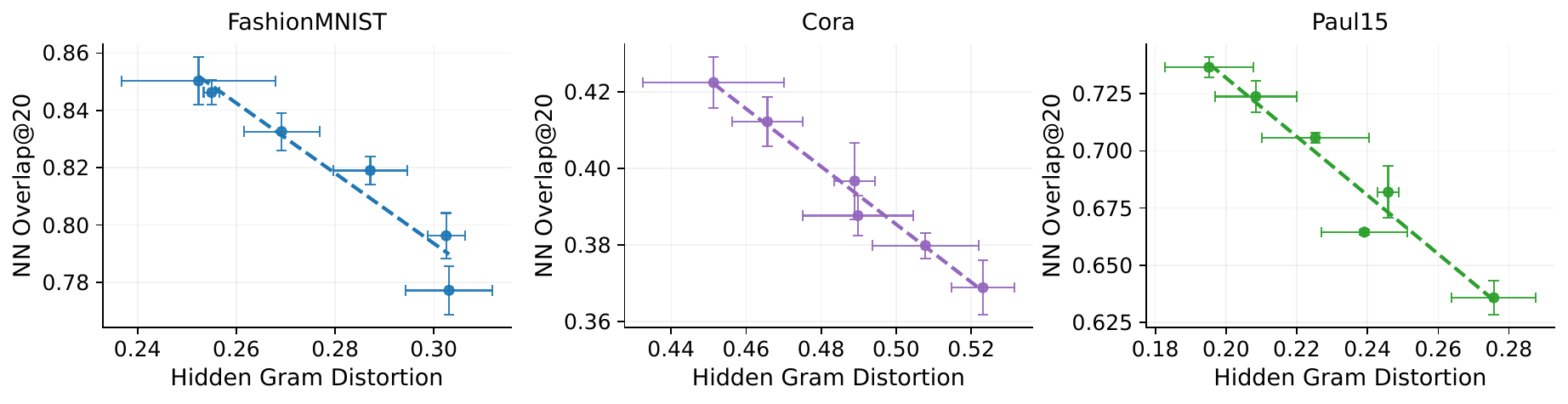}
    \caption{\textbf{Hidden Gram preservation predicts neighborhood preservation on real data.}
For each sparsification budget and dataset, we repeat the sparsifier with three random seeds and compute both hidden Gram distortion and mean $20$-NN overlap on held-out test nodes. Horizontal error bars show one standard deviation of hidden Gram distortion, and vertical error bars show one standard deviation of $20$-NN overlap. Across FashionMNIST, Cora, and Paul15, worse hidden Gram distortion is associated with worse neighborhood preservation, suggesting geometry preservation is a useful representation-fidelity metric.}
    \label{fig:hidden_to_nn}
\end{figure}

Our main practical question is whether preserving hidden Gram geometry also preserves useful local and class-level structure in representation space. Since the Gram matrix controls pairwise inner products, Theorem~\ref{thm:gram_stability} and Corollary~\ref{cor:distance_preservation} predict stability of pairwise embedding geometry. Corollary~\ref{cor:class_stats} further predicts stability of coarse class statistics such as centroids and within-class spread. For real datasets, we therefore focus directly on representation-level distortion, using hidden Gram distortion as the geometry-preservation metric.

Figure~\ref{fig:hidden_to_nn} tests the local-structure side of this prediction. The $20$-NN overlap measures the fraction of each node's nearest neighbors in the dense embedding that remain nearest neighbors after sparsification, averaged over held-out nodes. \camchange{At zero distortion the overlap is $1$ by construction, providing the natural no-perturbation baseline.} Across FashionMNIST, Cora, and Paul15, larger hidden Gram distortion is consistently associated with lower $20$-NN overlap. The relationship is strongly monotone on all three datasets, with Spearman correlations $-0.92$, $-0.90$, and $-0.92$, respectively. Because $20$-NN overlap measures preservation of each node's local embedding neighborhood, these results show that Gram distortion tracks a practically meaningful notion of embedding fidelity.

Figure~\ref{fig:centroids_main} tests the class-level side of the same picture. For each dataset, we fit a two-dimensional PCA projection on the dense embedding and project both dense and sparse class centroids into that same basis. Across all three datasets, sparse centroids remain close to their dense counterparts, indicating that sparsification preserves coarse class organization in addition to local neighborhoods.

Appendix~\ref{app:extra_figures} provides an additional Procrustes-aligned embedding visualization, and Appendix~\ref{app:teacher_student} evaluates another downstream diagnostic for representation geometry preservation: a knowledge-distillation task in which a dense-graph teacher provides soft labels for a student trained on the sparsified graph.

\begin{figure}[t]
    \centering
    \includegraphics[width=\linewidth]{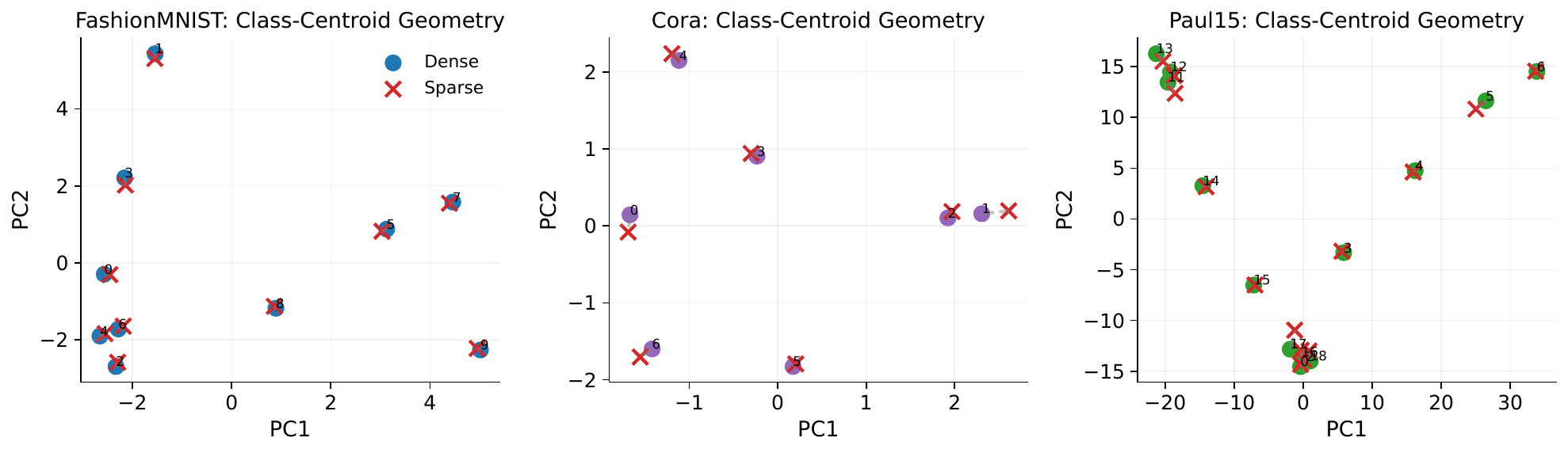}
    \caption{\textbf{Class-centroid geometry under sparsification.}
    Within each dataset, both dense and sparse centroids are projected into the same PCA basis.
    Small centroid displacements indicate that sparsification preserves class geometry, consistent with the class-statistics stability in Corollary~\ref{cor:class_stats}.}
    \label{fig:centroids_main}
\end{figure}

\section{Conclusion}

We studied graph sparsification from a representation-geometric perspective. We proved that spectral sparsification preserves not only polynomial graph filters, but also multilayer hidden embeddings and their Gram matrices. These bounds imply stability of pairwise distances, class means, and covariance structure in embedding space. We also studied the training dynamics and established finite-time stability of gradient descent trajectories in controlled one-layer and multilayer settings.
Empirically, we showed that representation and training perturbations scale smoothly with empirical spectral distortion on synthetic graphs, and that hidden Gram preservation predicts neighborhood and class-geometry preservation on real datasets. Across our datasets, hidden Gram matrices, embedding-geometry diagnostics, and finite-time training trajectories exhibited substantially smaller deviations than the underlying graph-level spectral distortion. These results suggest that sparsification should be evaluated not only by efficiency or predictive accuracy, but also by the geometry it preserves. This points toward a representation-aware view of graph compression for interpretable graph learning.
\newline
\textbf{Limitations:}
Our results assume that the sparsified graph satisfies an $\epsilon$-spectral approximation, which, despite being widely used, may hold only approximately in practical implementations. The analysis also relies on boundedness, Lipschitz, and smoothness assumptions, and the training-dynamics guarantees apply to controlled finite-time full-batch settings. \camchange{The guarantees directly cover fixed polynomial graph operators (including GCN- and ChebNet-style propagation) but not learned, feature-dependent graph operators. Our experiments are deliberately controlled and modest in scale; evaluating larger benchmarks, additional architectures and sparsifiers, more aggressive compression, and nodewise rather than aggregate representation deviations remains important future work.} Finally, while we show that geometric quantities are preserved and relate them to several downstream diagnostics, a deeper understanding of how representation geometry controls generalization remains an important direction for future work.
\newline
\textbf{Broader impact:}
Geometric representation preserving sparsification may reduce the computational cost of graph learning and make large-scale graph models more accessible. However, sparsification may disproportionately affect low-degree, peripheral, or minority nodes, potentially degrading performance for underrepresented groups in sensitive applications. Spectral or representation-level guarantees do not by themselves ensure fairness, privacy, or robustness.

\paragraph{Code availability.}
\camchange{Code can be found at \url{https://github.com/chimeraki/Spectral-Sparsification-Preserves-Representation-Geometry}.}


\bibliography{references}
\bibliographystyle{plainnat}

\clearpage
\appendix

\section{Proofs}
\label{app:proofs}

This appendix contains full proofs of the results stated in Section~\ref{sec:main_results}. We retain the notation introduced there.

\subsection{Auxiliary bounds}
\label{app:auxiliary}

We first record a standard identity for matrix powers.

\begin{lemma}[Telescoping identity]
\label{lem:telescoping}
For any square matrices $A,B$ and any integer $r\ge 1$,
\[
A^r - B^r = \sum_{j=0}^{r-1} A^{r-1-j}(A-B)B^j.
\]
\end{lemma}

\begin{proof}
The claim follows by induction on $r$. The case $r=1$ is immediate. If it holds for $r$, then
\[
A^{r+1}-B^{r+1}
=
A(A^r-B^r)+(A-B)B^r,
\]
and substituting the induction hypothesis yields the result.
\end{proof}

We will also use the following consequence of spectral sparsification.

\begin{lemma}
\label{lem:operator_diff}
If $(1-\varepsilon)L \preceq \widetilde L \preceq (1+\varepsilon)L$ and $\|L\|_2\le B_L$, then
\[
\|L-\widetilde L\|_2 \le \varepsilon B_L.
\]
\end{lemma}

\begin{proof}
For any unit vector $x$,
\[
-\varepsilon x^\top Lx
\le x^\top(\widetilde L-L)x
\le \varepsilon x^\top Lx.
\]
Since $0\preceq L$ and $\|L\|_2\le B_L$, we have $x^\top Lx\le B_L\|x\|_2^2=B_L$. Hence
\[
\bigl|x^\top(\widetilde L-L)x\bigr|\le \varepsilon B_L
\]
for every unit vector $x$. Because $\widetilde L-L$ is symmetric,
\[
\|\widetilde L-L\|_2
=
\sup_{\|x\|_2=1}\bigl|x^\top(\widetilde L-L)x\bigr|
\le \varepsilon B_L.
\]
\end{proof}

\subsection{Proof of Theorem~\ref{thm:poly_stability}}
\label{app:poly_stability}

\begin{proof}[Proof of Theorem~\ref{thm:poly_stability}]
Let
\[
p(x)=\sum_{r=0}^d a_r x^r.
\]
Then
\[
p(L)-p(\widetilde L)=\sum_{r=1}^d a_r (L^r-\widetilde L^r).
\]
By Lemma~\ref{lem:telescoping},
\[
L^r-\widetilde L^r
=
\sum_{j=0}^{r-1} L^{\,r-1-j}(L-\widetilde L)\widetilde L^{\,j}.
\]
Taking operator norms and using $\|L\|_2,\|\widetilde L\|_2\le B_L$,
\[
\|L^r-\widetilde L^r\|_2
\le
\sum_{j=0}^{r-1}\|L\|_2^{r-1-j}\|L-\widetilde L\|_2\|\widetilde L\|_2^j
\le
r B_L^{\,r-1}\|L-\widetilde L\|_2.
\]
Applying Lemma~\ref{lem:operator_diff},
\[
\|L^r-\widetilde L^r\|_2
\le
r B_L^{\,r}\varepsilon.
\]
Therefore,
\[
\|p(L)-p(\widetilde L)\|_2
\le
\sum_{r=1}^d |a_r|\,\|L^r-\widetilde L^r\|_2
\le
\left(\sum_{r=1}^d |a_r|\,r\,B_L^{\,r}\right)\varepsilon.
\]
This is exactly \eqref{eq:poly_stability}.
\end{proof}

\subsection{Proof of Theorem~\ref{thm:multilayer_stability}}
\label{app:multilayer_stability}

\begin{proof}[Proof of Theorem~\ref{thm:multilayer_stability}]
For each layer define
\[
A_k := p_k(L),
\qquad
\widetilde A_k := p_k(\widetilde L).
\]
Since $\|L\|_2,\|\widetilde L\|_2\le B_L$ and each $p_k$ has fixed finite degree, there exist constants
\[
M_k := \sup_{\|S\|_2\le B_L}\|p_k(S)\|_2<\infty
\]
such that
\[
\|A_k\|_2,\|\widetilde A_k\|_2\le M_k.
\]
We first record uniform forward bounds. Let $B_0:=B_X$. Since $\sigma_k$ is Lipschitz,
\[
\|H^{(k+1)}\|_F
\le L_{\sigma_k}\|A_kH^{(k)}W_k\|_F+|\sigma_k(0)|\sqrt{nd_{k+1}}.
\]
Using $\|A_k\|_2\le M_k$ and $\|W_k\|_2\le B_W$, we get
\[
\|H^{(k+1)}\|_F
\le L_{\sigma_k}M_kB_W\|H^{(k)}\|_F+|\sigma_k(0)|\sqrt{nd_{k+1}}.
\]
The same bound holds for $\widetilde H^{(k+1)}$. Hence, by induction, there are constants $B_k'$ independent of $\varepsilon$ such that
\[
\|H^{(k)}\|_F,\|\widetilde H^{(k)}\|_F\le B_k'
\qquad\text{for all }k=0,\ldots,K.
\]
By Theorem~\ref{thm:poly_stability},
\[
\|A_k-\widetilde A_k\|_2\le C_{p_k}\varepsilon.
\]

Let
\[
\Delta_k:=\|H^{(k)}-\widetilde H^{(k)}\|_F.
\]
Since $H^{(0)}=\widetilde H^{(0)}=X$, we have $\Delta_0=0$.

Using the layer definitions, Lipschitz continuity of $\sigma_k$, and adding and subtracting $\widetilde A_kH^{(k)}W_k$, we obtain
\begin{align}
\Delta_{k+1}
&=\|\sigma_k(A_kH^{(k)}W_k)-\sigma_k(\widetilde A_k\widetilde H^{(k)}W_k)\|_F \notag\\
&\le L_{\sigma_k}\Bigl(\|(A_k-\widetilde A_k)H^{(k)}W_k\|_F
+\|\widetilde A_k(H^{(k)}-\widetilde H^{(k)})W_k\|_F\Bigr).
\label{eq:delta_recursion_step}
\end{align}
Applying the weight and operator bounds to \eqref{eq:delta_recursion_step} gives
\[
\Delta_{k+1}
\le L_{\sigma_k}B_W\Bigl(\|A_k-\widetilde A_k\|_2\|H^{(k)}\|_F+\|\widetilde A_k\|_2\Delta_k\Bigr).
\]
Using Theorem~\ref{thm:poly_stability} and the uniform forward bounds gives
\[
\Delta_{k+1}\le \alpha_k\varepsilon+\beta_k\Delta_k,
\]
where
\[
\alpha_k:=L_{\sigma_k}B_WC_{p_k}B_k',
\qquad
\beta_k:=L_{\sigma_k}B_WM_k.
\]
Since $\Delta_0=0$, iterating this finite recursion gives
\[
\Delta_K\le C_{\mathrm{rep}}\varepsilon,
\]
where $C_{\mathrm{rep}}$ depends only on the stated network and boundedness constants, which proves \eqref{eq:multilayer_stability}.
\end{proof}

\subsection{Proof of Theorem~\ref{thm:gram_stability}}
\label{app:gram_stability}

\begin{proof}
Let $\Delta:=Z-\widetilde Z$. Then
\[
ZZ^\top-\widetilde Z\widetilde Z^\top=Z\Delta^\top+\Delta\widetilde Z^\top.
\]
For the spectral norm,
\begin{align*}
\|ZZ^\top-\widetilde Z\widetilde Z^\top\|_2
&\le \|Z\Delta^\top\|_2+\|\Delta\widetilde Z^\top\|_2\\
&\le (\|Z\|_2+\|\widetilde Z\|_2)\|\Delta\|_F.
\end{align*}
For the Frobenius norm,
\begin{align*}
\|ZZ^\top-\widetilde Z\widetilde Z^\top\|_F
&\le \|Z\Delta^\top\|_F+\|\Delta\widetilde Z^\top\|_F\\
&\le (\|Z\|_2+\|\widetilde Z\|_2)\|\Delta\|_F.
\end{align*}
By the boundedness argument used in the proof of Theorem~\ref{thm:multilayer_stability}, both $\|Z\|_2$ and $\|\widetilde Z\|_2$ are bounded independently of $\varepsilon$. Theorem~\ref{thm:multilayer_stability} gives $\|\Delta\|_F\le C_{\mathrm{rep}}\varepsilon$. Therefore both claimed bounds follow.
\end{proof}

\subsection{Proofs of Corollaries~\ref{cor:distance_preservation} and~\ref{cor:class_stats}}
\label{app:geometry_corollaries}

\begin{proof}[Proof of Corollary~\ref{cor:distance_preservation}]
Recall that
\[
\|z_i-z_j\|_2^2
=
\langle z_i,z_i\rangle+\langle z_j,z_j\rangle-2\langle z_i,z_j\rangle.
\]
The same identity holds for $\widetilde z_i,\widetilde z_j$. Therefore,
\begin{align*}
\bigl|
\|z_i-z_j\|_2^2-\|\widetilde z_i-\widetilde z_j\|_2^2
\bigr|
&\le
|\langle z_i,z_i\rangle-\langle \widetilde z_i,\widetilde z_i\rangle|
+
|\langle z_j,z_j\rangle-\langle \widetilde z_j,\widetilde z_j\rangle| \\
&\quad +
2|\langle z_i,z_j\rangle-\langle \widetilde z_i,\widetilde z_j\rangle|.
\end{align*}
Each inner-product difference is an entry of the matrix $ZZ^\top-\widetilde Z\widetilde Z^\top$, hence is bounded by its operator norm. Applying Theorem~\ref{thm:gram_stability} gives the result.
\end{proof}

\begin{proof}[Proof of Corollary~\ref{cor:class_stats}]
Let $\Delta := Z-\widetilde Z$, and let $\Delta_i$ denote its $i$th row. Then
\[
\mu_c-\widetilde\mu_c
=
\frac{1}{n_c}\sum_{i\in S_c}\Delta_i.
\]
By Cauchy--Schwarz,
\[
\|\mu_c-\widetilde\mu_c\|_2
\le
\frac{1}{n_c}\sum_{i\in S_c}\|\Delta_i\|_2
\le
\frac{1}{\sqrt{n_c}}
\left(\sum_{i\in S_c}\|\Delta_i\|_2^2\right)^{1/2}
\le
\frac{1}{\sqrt{n_c}}\|\Delta\|_F.
\]
This proves \eqref{eq:mean_stability}.

For the covariance bound, define
\[
a_i := z_i-\mu_c,
\qquad
\widetilde a_i := \widetilde z_i-\widetilde\mu_c.
\]
Then
\[
\Sigma_c-\widetilde\Sigma_c
=
\frac{1}{n_c}\sum_{i\in S_c}(a_i a_i^\top-\widetilde a_i \widetilde a_i^\top).
\]
Using
\[
a_i a_i^\top-\widetilde a_i \widetilde a_i^\top
=
a_i(a_i-\widetilde a_i)^\top + (a_i-\widetilde a_i)\widetilde a_i^\top,
\]
we obtain
\[
\|a_i a_i^\top-\widetilde a_i \widetilde a_i^\top\|_2
\le
(\|a_i\|_2+\|\widetilde a_i\|_2)\|a_i-\widetilde a_i\|_2.
\]
If all row norms of $Z$ and $\widetilde Z$ are bounded by $B_Z$, then
\[
\|a_i\|_2,\|\widetilde a_i\|_2 \le 2B_Z.
\]
Moreover,
\[
\|a_i-\widetilde a_i\|_2
\le
\|z_i-\widetilde z_i\|_2+\|\mu_c-\widetilde\mu_c\|_2.
\]
Therefore,
\[
\|\Sigma_c-\widetilde\Sigma_c\|_2
\le
\frac{4B_Z}{n_c}\sum_{i\in S_c}
\left(\|z_i-\widetilde z_i\|_2+\|\mu_c-\widetilde\mu_c\|_2\right).
\]
Using Cauchy--Schwarz for the first term and the mean bound for the second term,
\[
\frac{1}{n_c}\sum_{i\in S_c}\|z_i-\widetilde z_i\|_2
\le
\frac{1}{\sqrt{n_c}}\|\Delta\|_F,
\qquad
\|\mu_c-\widetilde\mu_c\|_2
\le
\frac{1}{\sqrt{n_c}}\|\Delta\|_F.
\]
Combining these yields
\[
\|\Sigma_c-\widetilde\Sigma_c\|_2
\le
\frac{8B_Z}{\sqrt{n_c}}\|\Delta\|_F.
\]
Finally, Theorem~\ref{thm:multilayer_stability} gives $\|\Delta\|_F=O(\varepsilon)$.
\end{proof}

\subsection{Proof of Theorem~\ref{thm:training_stability}}
\label{app:training_stability}

We first establish two Lipschitz properties of the gradient.

\begin{lemma}[Gradient Lipschitz in weights]
\label{lem:grad_lipschitz_W}
Under (A1), (A5), (A6), and (A7), there exists $L_W>0$ such that
\[
\|\nabla_W J(W;L)-\nabla_W J(V;L)\|_F
\le
L_W \|W-V\|_F
\]
for all $W,V$ in the bounded region $\{\|W\|_2,\|V\|_2\le R\}$.
\end{lemma}

\begin{proof}
Let $S:=p(L)$. The gradient can be written as
\[
\nabla_W J(W;L)
=
X^\top S\Bigl((\sigma(SXW)-Y)\odot \sigma'(SXW)\Bigr).
\]
This expression is a composition of affine maps, matrix products, and pointwise $C^1$ nonlinearities. The Lipschitz assumption on $\sigma'$ ensures that these compositions have uniformly bounded first derivatives on the prescribed bounded parameter/state region. Under assumptions (A4)--(A8), all quantities appearing in the gradient remain in a bounded set, and on bounded sets both $\sigma$ and $\sigma'$ are Lipschitz. Hence the map $W\mapsto \nabla_W J(W;L)$ is locally Lipschitz on that bounded set. Therefore there exists $L_W>0$ such that
\[
\|\nabla_W J(W;L)-\nabla_W J(V;L)\|_F
\le
L_W\|W-V\|_F.
\]
\end{proof}

\begin{lemma}[Gradient Lipschitz in the graph operator]
\label{lem:grad_lipschitz_S}
Under (A1), (A4)--(A7), there exists $L_S>0$ such that
\[
\|\nabla_W J(W;L)-\nabla_W J(W;\widetilde L)\|_F
\le
L_S \|p(L)-p(\widetilde L)\|_2
\]
for all $W$ with $\|W\|_2\le R$.
\end{lemma}

\begin{proof}
Writing $S:=p(L)$ and $\widetilde S:=p(\widetilde L)$, the same gradient formula shows that $\nabla_W J(W;\cdot)$ is locally Lipschitz in the propagation operator on bounded sets. Since $W$ remains bounded and $\sigma,\sigma'$ are bounded and Lipschitz on bounded sets, there exists $L_S>0$ such that
\[
\|\nabla_W J(W;L)-\nabla_W J(W;\widetilde L)\|_F
\le
L_S\|p(L)-p(\widetilde L)\|_2.
\]
\end{proof}

\begin{proof}[Proof of Theorem~\ref{thm:training_stability}]
Starting from the gradient descent updates,
\[
W_{t+1}-\widetilde W_{t+1}
=
W_t-\widetilde W_t
-
\eta\Bigl(
\nabla_W J(W_t;L)-\nabla_W J(\widetilde W_t;\widetilde L)
\Bigr).
\]
Taking Frobenius norms and using the triangle inequality,
\[
\|W_{t+1}-\widetilde W_{t+1}\|_F
\le
\|W_t-\widetilde W_t\|_F
+
\eta
\|\nabla_W J(W_t;L)-\nabla_W J(\widetilde W_t;\widetilde L)\|_F.
\]
Split the gradient difference as
\begin{align*}
\nabla_W J(W_t;L)-\nabla_W J(\widetilde W_t;\widetilde L)
&=
\bigl(
\nabla_W J(W_t;L)-\nabla_W J(\widetilde W_t;L)
\bigr) \\
&\quad +
\bigl(
\nabla_W J(\widetilde W_t;L)-\nabla_W J(\widetilde W_t;\widetilde L)
\bigr).
\end{align*}
Applying Lemmas~\ref{lem:grad_lipschitz_W} and \ref{lem:grad_lipschitz_S},
\[
\|W_{t+1}-\widetilde W_{t+1}\|_F
\le
(1+\eta L_W)\|W_t-\widetilde W_t\|_F
+
\eta L_S \|p(L)-p(\widetilde L)\|_2.
\]
This proves \eqref{eq:theorem_training_one_layer_main}. Since $W_0=\widetilde W_0$, unrolling the recursion yields
\[
\|W_t-\widetilde W_t\|_F
\le
\eta L_S \|p(L)-p(\widetilde L)\|_2
\sum_{j=0}^{t-1}(1+\eta L_W)^j.
\]
Finally, Theorem~\ref{thm:poly_stability} implies
\[
\|p(L)-p(\widetilde L)\|_2 \le C_p \varepsilon,
\]
which gives the claimed bound.
\end{proof}

\subsection{Proof of Proposition~\ref{prop:multilayer_training_stability}}
\label{app:multilayer_training_stability}

In this subsection we prove the multilayer analogue of Theorem~\ref{thm:training_stability}. For a parameter tuple
\[
\Theta=(W_0,\dots,W_{K-1}),
\]
define the hidden states recursively by
\begin{align}
H^{(0)}&=X, \notag\\
H^{(k+1)}&=\sigma_k\!\bigl(S_k H^{(k)}W_k\bigr),
\qquad
S_k:=p_k(L),
\label{eq:app_multilayer_forward_dense}
\end{align}
and similarly, on the sparsified graph,
\begin{align}
\widetilde H^{(0)}&=X, \notag\\
\widetilde H^{(k+1)}
&=
\sigma_k\!\bigl(\widetilde S_k \widetilde H^{(k)}W_k\bigr),
\qquad
\widetilde S_k:=p_k(\widetilde L).
\label{eq:app_multilayer_forward_sparse}
\end{align}
Let
\[
F(\Theta;L):=H^{(K)},
\qquad
J(\Theta;L):=\frac12 \|F(\Theta;L)-Y\|_F^2,
\]
and equip parameter space with the product Frobenius norm
\[
\|\Theta\|_F^2:=\sum_{k=0}^{K-1}\|W_k\|_F^2.
\]

For each layer $k$, define the preactivation
\[
U_k:=S_k H^{(k)}W_k.
\]
We also introduce the backpropagated state gradients
\[
\Delta_K := H^{(K)}-Y,
\]
and for $k=K-1,\dots,0$,
\begin{equation}
\Delta_k
=
S_k^\top
\Bigl(
\bigl(\Delta_{k+1}\odot \sigma_k'(U_k)\bigr)W_k^\top
\Bigr).
\label{eq:app_delta_recursion}
\end{equation}
With this notation, the gradient with respect to $W_k$ is
\begin{equation}
\nabla_{W_k}J(\Theta;L)
=
(H^{(k)})^\top S_k^\top
\bigl(\Delta_{k+1}\odot \sigma_k'(U_k)\bigr).
\label{eq:app_multilayer_weight_gradient}
\end{equation}

We first prove two Lipschitz properties of the full gradient.

\begin{lemma}[Gradient Lipschitz in the parameters]
\label{lem:multilayer_grad_lipschitz_theta}
Under assumptions (A1)--(A5), (A6$'$), (A7$'$), and (A8) there exists a constant $L_\Theta>0$ such that
\[
\|\nabla_\Theta J(\Theta;L)-\nabla_\Theta J(\Phi;L)\|_F
\le
L_\Theta \|\Theta-\Phi\|_F
\]
for all parameter tuples $\Theta,\Phi$ in the bounded region determined by (A7$'$).
\end{lemma}

\begin{proof}
From the forward recursion \eqref{eq:app_multilayer_forward_dense}, the backward recursion \eqref{eq:app_delta_recursion}, and the gradient formula \eqref{eq:app_multilayer_weight_gradient}, all forward states, backward states, and gradient blocks are compositions of affine maps, matrix products, and pointwise $C^1$ nonlinearities. Under assumptions (A1)--(A5), (A6$'$), (A7$'$), and (A8) these quantities remain in a bounded set, on which all such maps are Lipschitz. Therefore the full gradient map $\Theta\mapsto \nabla_\Theta J(\Theta;L)$ is locally Lipschitz on that bounded set, so there exists $L_\Theta>0$ such that
\[
\|\nabla_\Theta J(\Theta;L)-\nabla_\Theta J(\Phi;L)\|_F
\le
L_\Theta \|\Theta-\Phi\|_F.
\]
\end{proof}

\begin{lemma}[Gradient Lipschitz in the graph filters]
\label{lem:multilayer_grad_lipschitz_filters}
Under assumptions (A1)--(A5), (A6$'$), and (A7$'$), there exists a constant $L_S>0$ such that
\[
\|\nabla_\Theta J(\Theta;L)-\nabla_\Theta J(\Theta;\widetilde L)\|_F
\le
L_S \sum_{k=0}^{K-1}\|S_k-\widetilde S_k\|_2
\]
for all parameter tuples $\Theta$ in the bounded region determined by (A7$'$).
\end{lemma}

\begin{proof}
Fix $\Theta$. Using the same parameter tuple on both graphs, the only perturbation comes from replacing $\{S_k\}_{k=0}^{K-1}$ by $\{\widetilde S_k\}_{k=0}^{K-1}$. The forward states, backward states, and gradient blocks are again compositions of affine maps, matrix products, and pointwise $C^1$ nonlinearities, and remain bounded under (A1)--(A5), (A6$'$), and (A7$'$). Hence the full gradient is locally Lipschitz in the collection of graph filters on bounded sets. Therefore there exists $L_S>0$ such that
\[
\|\nabla_\Theta J(\Theta;L)-\nabla_\Theta J(\Theta;\widetilde L)\|_F
\le
L_S \sum_{k=0}^{K-1}\|S_k-\widetilde S_k\|_2.
\]
\end{proof}

\begin{proof}[Proof of Proposition~\ref{prop:multilayer_training_stability}]
This is the same perturbation argument as in Theorem~\ref{thm:training_stability}, now applied in product parameter space and with the graph perturbation measured through the collection $\{p_k(L)-p_k(\widetilde L)\}_{k=0}^{K-1}$.

Starting from the gradient descent updates,
\begin{equation}
\Theta^{t+1}-\widetilde\Theta^{t+1}
=
\Theta^t-\widetilde\Theta^t
-
\eta\Bigl(
\nabla_\Theta J(\Theta^t;L)
-
\nabla_\Theta J(\widetilde\Theta^t;\widetilde L)
\Bigr).
\label{eq:app_multilayer_gd_difference}
\end{equation}
Taking Frobenius norms and using the triangle inequality,
\begin{equation}
\|\Theta^{t+1}-\widetilde\Theta^{t+1}\|_F
\le
\|\Theta^t-\widetilde\Theta^t\|_F
+
\eta
\|\nabla_\Theta J(\Theta^t;L)
-
\nabla_\Theta J(\widetilde\Theta^t;\widetilde L)\|_F.
\label{eq:app_multilayer_basic_recursion}
\end{equation}
Split the gradient difference as
\begin{align}
\nabla_\Theta J(\Theta^t;L)-\nabla_\Theta J(\widetilde\Theta^t;\widetilde L)
&=
\Bigl(
\nabla_\Theta J(\Theta^t;L)
-
\nabla_\Theta J(\widetilde\Theta^t;L)
\Bigr) \notag\\
&\quad+
\Bigl(
\nabla_\Theta J(\widetilde\Theta^t;L)
-
\nabla_\Theta J(\widetilde\Theta^t;\widetilde L)
\Bigr).
\label{eq:app_multilayer_gradient_split}
\end{align}
Applying Lemmas~\ref{lem:multilayer_grad_lipschitz_theta} and~\ref{lem:multilayer_grad_lipschitz_filters} gives
\[
\|\Theta^{t+1}-\widetilde\Theta^{t+1}\|_F
\le
(1+\eta L_\Theta)\|\Theta^t-\widetilde\Theta^t\|_F
+
\eta L_S
\sum_{k=0}^{K-1}\|S_k-\widetilde S_k\|_2.
\]
This proves \eqref{eq:theorem_training_multi_layer_main}.

Since $\Theta^0=\widetilde\Theta^0$, unrolling the recursion yields
\begin{equation}
\label{eq:theta}
\|\Theta^t-\widetilde\Theta^t\|_F
\le
\eta L_S
\left(
\sum_{k=0}^{K-1}\|S_k-\widetilde S_k\|_2
\right)
\sum_{j=0}^{t-1}(1+\eta L_\Theta)^j.
\end{equation}
Finally, by Theorem~\ref{thm:poly_stability}, for each $k$,
\[
\|S_k-\widetilde S_k\|_2
=
\|p_k(L)-p_k(\widetilde L)\|_2
\le
C_{p_k}\varepsilon.
\]
Therefore
\[
\sum_{k=0}^{K-1}\|S_k-\widetilde S_k\|_2
\le
\left(\sum_{k=0}^{K-1}C_{p_k}\right)\varepsilon,
\]
Substituting this estimate into Eqn.~\ref{eq:theta} gives
\[
\|\Theta^t-\widetilde\Theta^t\|_F
\le
\eta L_S
\left(\sum_{k=0}^{K-1} C_{p_k}\right)
\varepsilon
\sum_{j=0}^{t-1}(1+\eta L_\Theta)^j,
\]
which proves the claimed bound with
\[
C_{\mathrm{train}}^{(K)}
=
\eta L_S\sum_{k=0}^{K-1} C_{p_k}.
\]
\end{proof}

\section{Additional Experimental Details}
\label{app:experiments}

\subsection{Dataset construction}
\label{app:datasets}

\paragraph{Synthetic SBM.}
We generate a four-block weighted stochastic block model with block sizes \((80,80,80,80)\), within-block edge probability \(0.20\), and cross-block edge probability \(0.05\).

\paragraph{Synthetic geometric graph.}
We generate four Gaussian feature clusters with \(80\) nodes per class and \(20\)-dimensional features. A weighted \(k\)-NN graph with \(k=30\) is built from these features using Gaussian kernel weights.

\paragraph{FashionMNIST.}
We load the training split, subsample $220$ examples per class in the full setting, flatten the images, reduce to $50$ PCA dimensions, and build a weighted $15$-NN graph.

\paragraph{Cora.}
We use the Planetoid Cora benchmark and take the citation graph as the adjacency.

\paragraph{Paul15.}
We use the Paul15 single-cell dataset from Scanpy. Labels are extracted from an available cluster annotation, features are reduced to at most $64$ dimensions when needed, and a weighted $15$-NN graph is built from expression features.

\paragraph{Splits.}
Class-balanced random masks are used throughout. For FashionMNIST, we use $20$ train and $40$ validation examples per class. For Cora, we use $20$ train and $50$ validation examples per class. For Paul15, we use $20$ train and $30$ validation examples per class. Remaining nodes are used for testing. On synthetic datasets, analogous class-balanced masks are used.

\subsection{Architectures and optimization}
\label{app:impl}
\paragraph{Synthetic theory-facing models.}
For the controlled sparsification experiments of Figure~\ref{fig:theory_main}, we use a fixed deterministic two-stage polynomial-filter map so that the forward distortion induced by sparsification can be measured without confounding from stochastic optimization. We use the scaled combinatorial Laplacian
\[
\bar L = L/\|L\|_2
\]
as the graph operator, and apply the same dense-graph scaling to the sparsified Laplacian. The polynomial filter is
\[
p(\bar L)=I-0.6\bar L+0.15\bar L^2.
\]
The map consists of this polynomial graph filter followed by two deterministic feature-mixing stages.

For the training-dynamics experiments of Figure~\ref{fig:training_dynamics_main}, we use the same polynomial-Laplacian graph shift. Dense and sparse trajectories differ only through replacing \(p(\bar L)\) by \(p(\widetilde{\bar L})\), with matched initialization.

For the training-dynamics plots, the number of training epochs is $100$. The learning rates are $0.01$ for the one-layer model and $0.003$ for the two-layer model, with weight decay $10^{-3}$ and gradient clipping.

\paragraph{Teacher models.}
For the synthetic practical experiments, the teacher is a deterministic two-layer GCN trained on the dense graph. The hidden dimension is $48$ for the SBM and $32$ for the geometric graph. Teachers are trained for $120$ epochs with Adam, learning rate $0.01$, and weight decay $5\times 10^{-4}$.

For the real-data experiments, we use a two-layer GCN with ReLU activation and no dropout. When the input feature dimension is at most $64$, we use hidden dimension $48$, which is the case for the datasets considered here. In the full runs, teachers are trained for $20$ epochs with Adam, learning rate $0.01$, and weight decay $5\times 10^{-4}$.

\subsection{Sparsification details}
\label{app:sparsifier}

\paragraph{Synthetic graphs.}
In the controlled synthetic experiments, we use exact effective-resistance sampling on the weighted SBM and geometric graphs. The sampling budget is the number of edge samples drawn with replacement before reweighting. We parameterize it as
\[
q=\lfloor c\,n\log n\rfloor,
\]
where larger $c$ draws more samples and typically yields a denser, lower-distortion sparsifier.
Rather than reporting only a fixed budget sweep, we search over a grid of budget multipliers $c$ and, for each prescribed target level, select a sparsified graph whose realized empirical distortion $\epsilon_{\mathrm{emp}}$ is closest to that target. Each selected target level is then repeated with multiple sparsifier draws. In the base setting we use $5$ repetitions per selected target level.

For each sparsified graph $\widetilde G$, we compute the empirical graph distortion $\epsilon_{\mathrm{emp}}$ from the generalized eigenvalue envelope of the dense and sparse Laplacians on the orthogonal complement of the all-ones vector. Specifically, if
\[
\lambda_{\min} \le \frac{x^\top \widetilde L x}{x^\top L x} \le \lambda_{\max}
\qquad \text{for all } x \perp \mathbf{1},
\]
then we report
\[
\epsilon_{\mathrm{emp}} := \max\{1-\lambda_{\min},\,\lambda_{\max}-1\}.
\]

\paragraph{Real datasets.}
For real-data experiments, we use approximate effective-resistance sampling, where resistance scores are estimated using a truncated Laplacian eigenspace. In the full runs, the approximation rank is $16$, the budget multipliers are
\begin{equation}
c \in \{0.45,0.60,0.80,1.00,1.25,1.50\},
\end{equation}
and each budget is repeated $3$ times with different sparsifier seeds. The empirical graph distortion is approximated by probing quadratic-form distortion with random test vectors.

\subsection{Metric definitions}
\label{app:metrics}

On synthetic graphs, we report filter error, relative hidden representation error, and relative hidden Gram error. On real data, the primary geometry metric is relative hidden Gram distortion,
\begin{equation}
D_{\mathrm{Gram}}(Z,\widetilde Z)
=
\frac{\|ZZ^\top-\widetilde Z\widetilde Z^\top\|_F}
{\|ZZ^\top\|_F}.
\end{equation}

To measure local geometry preservation, we compute mean $20$-nearest-neighbor overlap on held-out nodes. 

Let $z_i^\top$ denote the $i$th row of $Z$, i.e., the hidden embedding of node $i$. We define the similarity between nodes $i$ and $j$ by the cosine similarity of their hidden embeddings,
\[
\mathrm{cos}(z_i,z_j) := \frac{\langle z_i,z_j\rangle}{\|z_i\|_2\|z_j\|_2}.
\]
Let $N_k(i;Z)$ denote the set of the $k$ nodes with largest cosine similarity to node $i$ in the embedding $Z$. The nearest neighbor overlap is given by
\begin{equation}
D_{\mathrm{NN}@k}(Z,\widetilde Z)
=
\frac{1}{|S|}
\sum_{i\in S}
\frac{|N_k(i;Z)\cap N_k(i;\widetilde Z)|}{k},
\end{equation}
with $k=20$ in all reported experiments. In the real-data script, these geometry metrics are evaluated on a random subsample of at most $500$ held-out test nodes to keep the computation lightweight. 

\paragraph{Compute resources.}
All experiments were run on an HP Spectre x360 laptop with an Intel Core i7-10750H CPU, 16\,GB RAM, and an NVIDIA GeForce GTX 1650 Ti with Max-Q Design GPU (4\,GB GDDR6 VRAM). In practice, the provided code ran primarily on CPU, with approximate runtimes of 5–15 minutes per real dataset and about 20–40 minutes for the full synthetic suite.

\subsection{Additional qualitative and downstream figures}
\label{app:extra_figures}

Figure~\ref{fig:procrustes_app} provides a qualitative view of the dense and sparse hidden embeddings after projection to two principal components and alignment by an optimal orthogonal Procrustes transform. The substantial overlap after alignment indicates that much of the dense--sparse discrepancy is explained by an approximately rigid transformation rather than severe geometric collapse.

\begin{figure}[htbp!]
    \centering
    \includegraphics[width=0.7\linewidth]{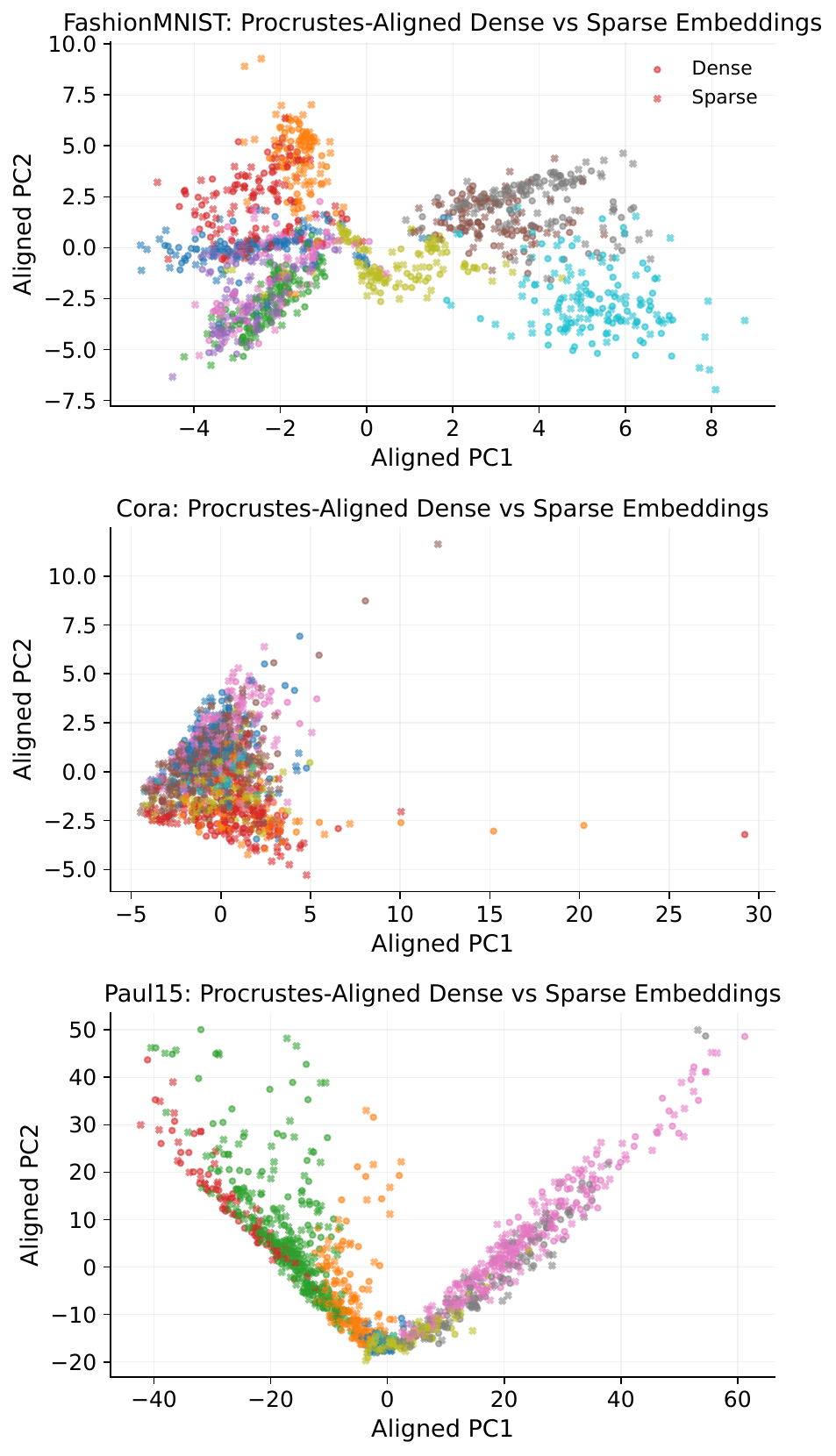}
    \caption{\textbf{Procrustes-aligned dense and sparse embeddings.}
    Dense and sparse hidden embeddings on held-out nodes after projection to two principal components and optimal orthogonal alignment of the sparse embedding to the dense one. The substantial overlap after alignment indicates that much of the dense--sparse discrepancy is explained by an approximately rigid transformation rather than severe deformation of the learned representation geometry.}
    \label{fig:procrustes_app}
\end{figure}

\subsection{Knowledge distillation under graph distortion}
\label{app:teacher_student}
This experiment provides a downstream diagnostic complementary to the representation-geometry metrics in the main text.
Here, we study teacher--student transfer on the synthetic graphs. A teacher is first trained on the dense graph. We then sparsify the graph and train a student on the sparse graph using a mixture of cross-entropy loss and KL-based knowledge distillation from the dense teacher. Figure~\ref{fig:teacher_student_app} shows that larger empirical graph distortion leads to larger KL divergence in teacher behavior, and that larger teacher KL is associated with worse KD student accuracy. This provides a downstream interpretation of representation distortion: graph perturbations alter the soft predictive information available for transfer.

\begin{figure}[htbp!]
    \centering
    \includegraphics[width=\linewidth]{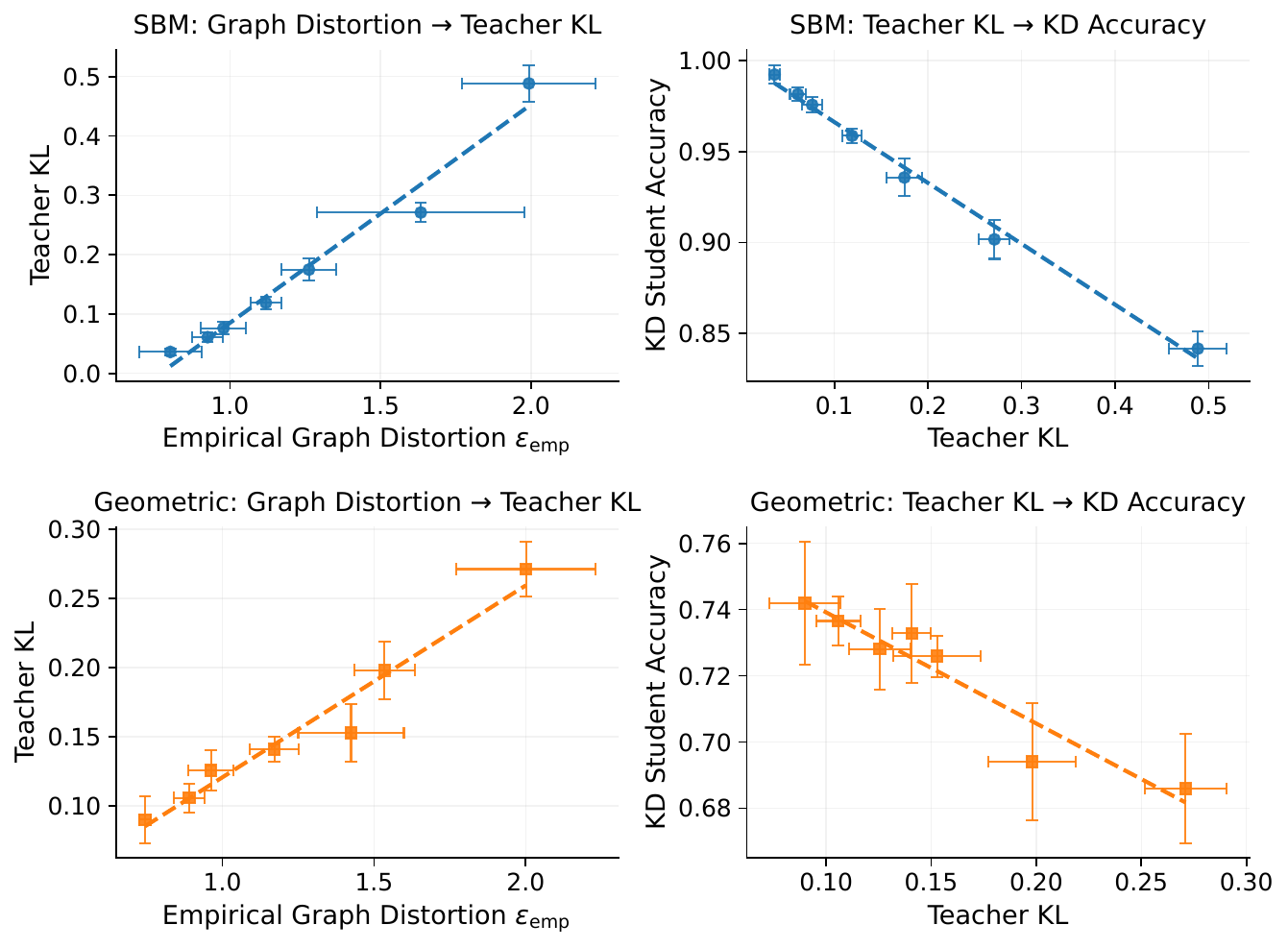}
    \caption{\textbf{Teacher--student consequences of graph distortion.}
    On synthetic SBM and geometric graphs, larger empirical graph distortion produces larger teacher KL divergence, and larger KL is associated with lower KD student accuracy. This shows that graph distortion affects not only hidden geometry but also the transferred soft decision structure. Error bars denote one standard deviation across four repeated runs.}
    \label{fig:teacher_student_app}
\end{figure}

\newpage
\section*{NeurIPS Paper Checklist}

\begin{enumerate}

\item {\bf Claims}
    \item[] Question: Do the main claims made in the abstract and introduction accurately reflect the paper's contributions and scope?
    \item[] Answer: \answerYes{}
    \item[] Justification: The abstract and introduction accurately state the paper’s main theoretical and empirical contributions: stability of polynomial graph filters, multilayer hidden representations, and Gram matrices under spectral sparsification; finite-time training-dynamics stability in simplified settings; and empirical evidence that representation-geometry preservation tracks practical embedding fidelity on synthetic and real datasets.

\item {\bf Limitations}
    \item[] Question: Does the paper discuss the limitations of the work performed by the authors?
    \item[] Answer: \answerYes{}
    \item[] Justification: The paper includes a dedicated limitations discussion in Section~4. It notes that the theory is proved for polynomial-filter GNNs under boundedness, Lipschitz, and smoothness assumptions; that training-dynamics guarantees are established only in simplified deterministic settings; that experiments instantiate the framework using effective-resistance-based sparsification as a concrete testbed; and that the empirical study does not yet cover a broader range of architectures, larger-scale benchmarks, alternative sparsifiers, or highly compressed regimes.

\item {\bf Theory assumptions and proofs}
    \item[] Question: For each theoretical result, does the paper provide the full set of assumptions and a complete (and correct) proof?
    \item[] Answer: \answerYes{}
    \item[] Justification: The assumptions are stated explicitly in Section~2, including (A1)--(A8) and the multilayer variants (A6$'$)--(A7$'$), and complete proofs are provided in Appendix~A for Theorems~1--4, Proposition~1, Corollaries~1--2, and the auxiliary lemmas they rely on.

    \item {\bf Experimental result reproducibility}
    \item[] Question: Does the paper fully disclose all the information needed to reproduce the main experimental results of the paper to the extent that it affects the main claims and/or conclusions of the paper (regardless of whether the code and data are provided or not)?
    \item[] Answer: \answerYes{}
    \item[] Justification: Sections~3.1--3.2 and Appendix~B specify the datasets, graph constructions, train/validation/test splits, model classes, perturbation levels, sparsification budgets, metrics, optimization settings, and repetition counts used for the reported experiments.

\item {\bf Open access to data and code}
    \item[] Question: Does the paper provide open access to the data and code, with sufficient instructions to faithfully reproduce the main experimental results, as described in supplemental material?
    \item[] Answer: \answerYes{}
    \item[] Justification: The paper provides detailed experimental specifications and a public code repository at \url{https://github.com/chimeraki/Spectral-Sparsification-Preserves-Representation-Geometry} with code for reproducing the reported experiments. Public datasets used in the experiments are standard benchmark datasets.

\item {\bf Experimental setting/details}
    \item[] Question: Does the paper specify all the training and test details (e.g., data splits, hyperparameters, how they were chosen, type of optimizer) necessary to understand the results?
    \item[] Answer: \answerYes{}
    \item[] Justification: The experimental setting is described in Section~3 and Appendix~B, including dataset construction, data splits, hidden dimensions, optimizers, learning rates, training epochs, sparsification budgets, and evaluation metrics.

\item {\bf Experiment statistical significance}
    \item[] Question: Does the paper report error bars suitably and correctly defined or other appropriate information about the statistical significance of the experiments?
    \item[] Answer: \answerYes{}
    \item[] Justification: The paper reports repeated runs and error bars where appropriate. In the full synthetic setting, the controlled perturbation experiments use 5 repetitions per perturbation level, and the synthetic knowledge-distillation experiment uses 4 repeated runs. In the full real-data setting, each sparsification budget is repeated 3 times with different sparsifier seeds. Reported error bars correspond to one standard deviation across repeated runs at fixed perturbation level or sparsification budget.

\item {\bf Experiments compute resources}
    \item[] Question: For each experiment, does the paper provide sufficient information on the computer resources (type of compute workers, memory, time of execution) needed to reproduce the experiments?
    \item[] Answer: \answerYes{}
    \item[] Justification: Appendix~B reports the hardware used for all experiments (HP Spectre x360 laptop, Intel Core i7-10750H CPU, 16\,GB RAM, NVIDIA GeForce GTX 1650 Ti with Max-Q Design GPU, 4\,GB GDDR6 VRAM), notes that the code ran primarily on CPU in practice, and provides order-of-magnitude runtime information for the real-data and synthetic experiment suites.
    
\item {\bf Code of ethics}
    \item[] Question: Does the research conducted in the paper conform, in every respect, with the NeurIPS Code of Ethics \url{https://neurips.cc/public/EthicsGuidelines}?
    \item[] Answer: \answerYes{}
    \item[] Justification: The work is theoretical and empirical on synthetic and standard benchmark datasets, does not involve human subjects, and does not introduce a deployment system or other high-risk application.

\item {\bf Broader impacts}
    \item[] Question: Does the paper discuss both potential positive societal impacts and negative societal impacts of the work performed?
    \item[] Answer: \answerYes{}
    \item[] Justification: The paper includes a broader-impacts paragraph discussing both potential benefits, such as improved efficiency and accessibility of graph-based learning methods, and possible risks if such methods are later deployed in sensitive decision-making settings without sufficient fairness, robustness, or privacy evaluation.
    
\item {\bf Safeguards}
    \item[] Question: Does the paper describe safeguards that have been put in place for responsible release of data or models that have a high risk for misuse (e.g., pre-trained language models, image generators, or scraped datasets)?
    \item[] Answer: \answerNA{}
    \item[] Justification: The paper does not release a high-risk model, scraped dataset, or other asset with an elevated misuse profile; its core contribution is theoretical and experimental analysis.

\item {\bf Licenses for existing assets}
    \item[] Question: Are the creators or original owners of assets (e.g., code, data, models), used in the paper, properly credited and are the license and terms of use explicitly mentioned and properly respected?
    \item[] Answer: \answerYes{}
    \item[] Justification: The paper cites the original benchmark datasets, papers, and prior methods used in the experiments. The datasets employed are standard public research assets.

\item {\bf New assets}
    \item[] Question: Are new assets introduced in the paper well documented and is the documentation provided alongside the assets?
    \item[] Answer: \answerNA{}
    \item[] Justification: The submission does not introduce a new dataset, benchmark, pretrained model, or standalone released asset as a core contribution.

\item {\bf Crowdsourcing and research with human subjects}
    \item[] Question: For crowdsourcing experiments and research with human subjects, does the paper include the full text of instructions given to participants and screenshots, if applicable, as well as details about compensation (if any)? 
    \item[] Answer: \answerNA{}
    \item[] Justification: The paper does not involve crowdsourcing or research with human subjects.

\item {\bf Institutional review board (IRB) approvals or equivalent for research with human subjects}
    \item[] Question: Does the paper describe potential risks incurred by study participants, whether such risks were disclosed to the subjects, and whether Institutional Review Board (IRB) approvals (or an equivalent approval/review based on the requirements of your country or institution) were obtained?
    \item[] Answer: \answerNA{}
    \item[] Justification: The paper does not involve human subjects.

\item {\bf Declaration of LLM usage}
    \item[] Question: Does the paper describe the usage of LLMs if it is an important, original, or non-standard component of the core methods in this research? Note that if the LLM is used only for writing, editing, or formatting purposes and does \emph{not} impact the core methodology, scientific rigor, or originality of the research, declaration is not required.
    \item[] Answer: \answerNA{}
    \item[] Justification: LLMs are not part of the paper’s core scientific methodology. Any LLM use was limited to writing, editing, formatting, and coding assistance rather than method development or experimental decision-making.

\end{enumerate}

\end{document}